\documentclass[runningheads]{llncs}

\usepackage{makeidx} 
\usepackage{amsmath,mathtools}
\usepackage{amssymb}
\usepackage{array}
\usepackage[bookmarks=false]{hyperref}
\usepackage{float}
\usepackage{graphicx}
\usepackage{macros}
\usepackage{subcaption}
\usepackage{multirow}
\usepackage{tikz}
\usepackage{color}
\usepackage{colortbl}
\usepackage{cancel}
\usepackage{relsize}
\usepackage{orcidlink}
\usepackage{shuffle}
\usepackage[vlined,boxed,linesnumbered]{algorithm2e}
\usepackage{svg}

\DeclareFontFamily{U}{mathb}{\hyphenchar\font45}
\DeclareFontShape{U}{mathb}{m}{n}{
      <5> <6> <7> <8> <9> <10>
      <10.95> <12> <14.4> <17.28> <20.74> <24.88>
      mathb10
      }{}
\DeclareSymbolFont{mathb}{U}{mathb}{m}{n}

\DeclareMathSymbol{\sqbullet}{1}{mathb}{"0D}

\makeatletter
\def\hlinewd#1{%
\noalign{\ifnum0=`}\fi\hrule \@height #1 %
\futurelet\reserved@a\@xhline}

\makeatletter
\def\old@comma{,}
\catcode`\,=13
\def,{%
  \ifmmode%
    \old@comma\discretionary{}{}{}%
  \else%
    \old@comma%
  \fi%
}
\makeatother

\begin{document}

\title{Unlocking Non-Block-Structured Decisions: Inductive Mining with Choice Graphs\thanks{The Version of Record of this contribution will be published in the proceedings of the 23rd International Conference on Business Process Management (BPM 2025). This preprint has not undergone peer review or any post-submission improvements or corrections.}}
\titlerunning{Inductive Mining with Choice Graphs}
\author{Humam Kourani\inst{1,2}\orcidID{0000-0003-2375-2152} \and
Gyunam Park\inst{1}\orcidID{0000-0001-9394-6513} \and
Wil M.P. van der Aalst\inst{1,2}\orcidID{0000-0002-0955-6940}}

\authorrunning{H. Kourani et al.}   

\institute{Fraunhofer Institute for Applied Information Technology FIT, Schloss Birlinghoven, 53757 Sankt Augustin, Germany\\
\email{\{humam.kourani,gyunam.park,wil.van.der.aalst\}@fit.fraunhofer.de} \and
RWTH Aachen University, Ahornstraße 55, 52074 Aachen, Germany}

\maketitle              
\begin{abstract}
Process discovery aims to automatically derive process models from event logs, enabling organizations to analyze and improve their operational processes.
Inductive mining algorithms, while prioritizing soundness and efficiency through hierarchical modeling languages, often impose a strict block-structured representation. This limits their ability to accurately capture the complexities of real-world processes.
While recent advancements like the Partially Ordered Workflow Language (POWL) have addressed the block-structure limitation for concurrency, a significant gap remains in effectively modeling non-block-structured decision points. 
In this paper, we bridge this gap by proposing an extension of POWL to handle non-block-structured decisions through the introduction of choice graphs. Choice graphs offer a structured yet flexible approach to model complex decision logic within the hierarchical framework of POWL. 
We present an inductive mining discovery algorithm that uses our extension and preserves the quality guarantees of the inductive mining framework.
Our experimental evaluation demonstrates that the discovered models, enriched with choice graphs, more precisely represent the complex decision-making behavior found in real-world processes, without compromising the high scalability inherent in inductive mining techniques.

\end{abstract}

\section{Introduction} \label{sec:intro}

Process discovery is a key branch of process mining that aims to automatically construct process models from event data, typically recorded in event logs. By uncovering the ``as-is'' process, organizations can gain insights into how their processes actually operate, identify bottlenecks, deviations from prescribed procedures, and opportunities for optimization.

A wide variety of process discovery approaches exist, each with its own strengths and weaknesses, particularly concerning scalability, model quality, and the ability to represent complex process behavior. The Inductive Miner \cite{sander} stands out as a prominent process discovery algorithm. A key characteristic of the Inductive Miner is its use of hierarchical process modeling languages for intermediate representation. This hierarchical nature offers several advantages. Firstly, the discovered models are \emph{sound} by construction, meaning that they are free from common workflow anomalies like deadlocks. Secondly, the hierarchical nature enables a recursive approach to discovery, ensuring efficient processing of large event logs. Furthermore, the generated hierarchical process models can be easily transformed into standard notations, such as BPMN \cite{DBLP:books/el/15/RosingWCM15} and Petri nets \cite{DBLP:journals/topnoc/HeeSW13a}, facilitating their use in a variety of analysis and implementation contexts.

The original Inductive Miner uses \emph{process trees} for intermediate representation. A process tree corresponds to a mathematical tree where leaves represent activities and internal node are control-flow operators that model the relationships between children. The base variant of the Inductive Miner supports four control-flow operators for modeling common process constructs: \textit{sequence}, \textit{exclusive choice}, \textit{concurrency}, and \textit{loop}. While process trees inherently guarantee soundness, their strict block-structured nature limits their expressiveness. The Inductive Miner only discovers well-nested blocks, meaning that every branching split (XOR/AND) has a corresponding matching join in the model. 

Let's consider an example of a retailer's order fulfillment process. This retailer handles two types of orders: those for in-stock items that the retailer sells but does not produce itself, and customized orders that require production. The ground truth process model for this process is shown in the BPMN notation in \autoref{fig:bpmn}. After receiving an order, the process starts with a decision point to determine the order type. The production sub-process follows a non-block-structured path involving gathering production materials, scheduling production, notifying the customer about the expected delivery date, and executing production. In this sub-process, notifying the customer is independent of gathering the production materials and executing the production, but it happens after scheduling the production so the delivery date can be estimated. Executing the production happens after both gathering the production materials and scheduling the production. This type of non-block-structured concurrency cannot be modeled in a process tree. For example, the tree shown in \autoref{fig:tree} is the result of applying the base Inductive Miner to an event log that simulates this process. This process tree overgeneralizes the process, allowing production to be executed before scheduling it or collecting the required materials.

\begin{figure}[!t]
\centering
    \begin{subfigure}{\textwidth}
        \centering
        \includegraphics[width=\textwidth]{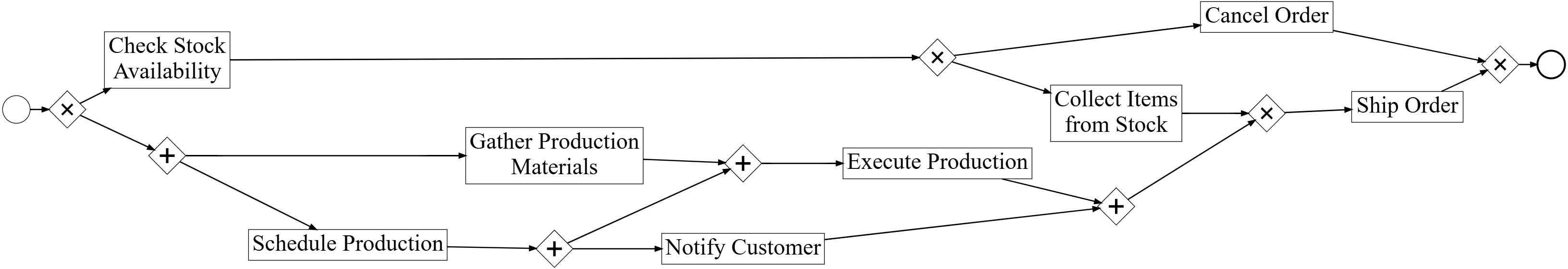}
        \caption{Ground truth BPMN model.}
        \label{fig:bpmn}
    \end{subfigure}
   
    \begin{subfigure}{\textwidth}
        \centering  \includegraphics[width=0.8\textwidth]{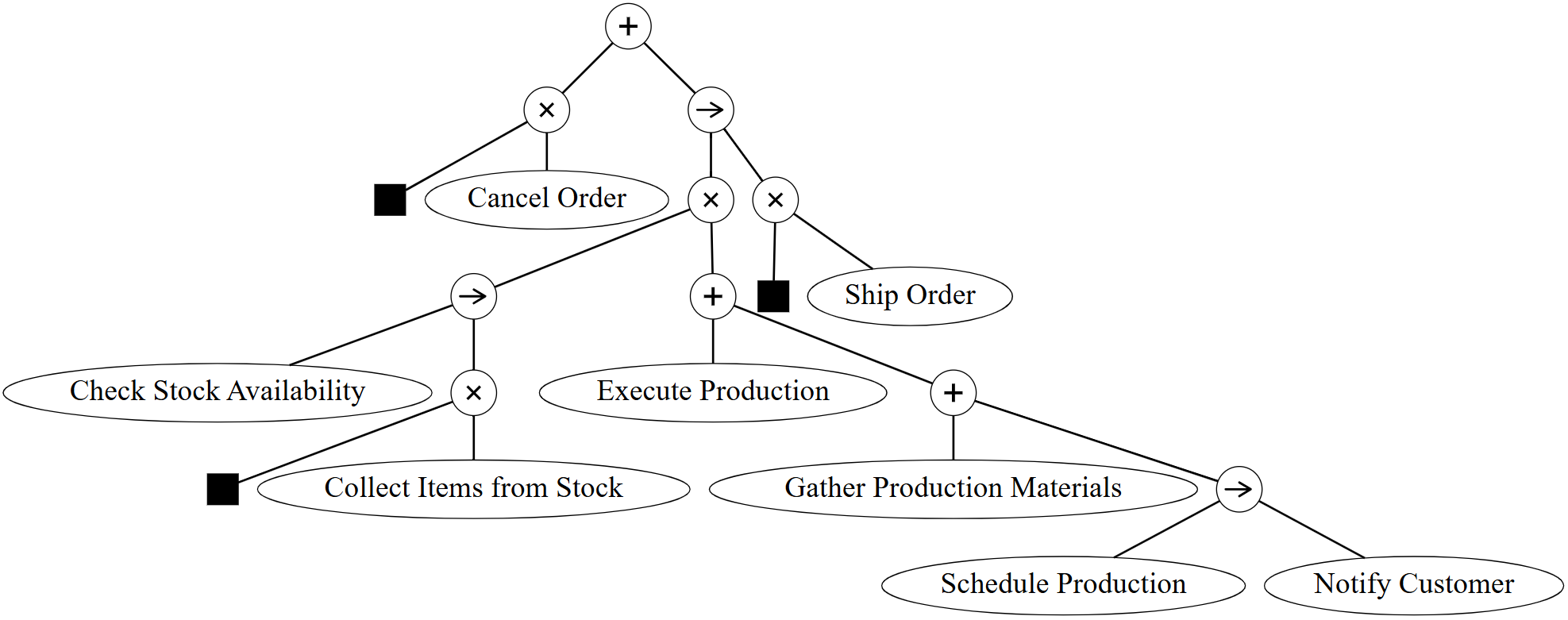}
        \caption{Process tree discovered with the base Inductive Miner \cite{sander}.}
        \label{fig:tree}
    \end{subfigure}

    \begin{subfigure}{\textwidth}
        \centering      \includegraphics[width=0.85\textwidth]{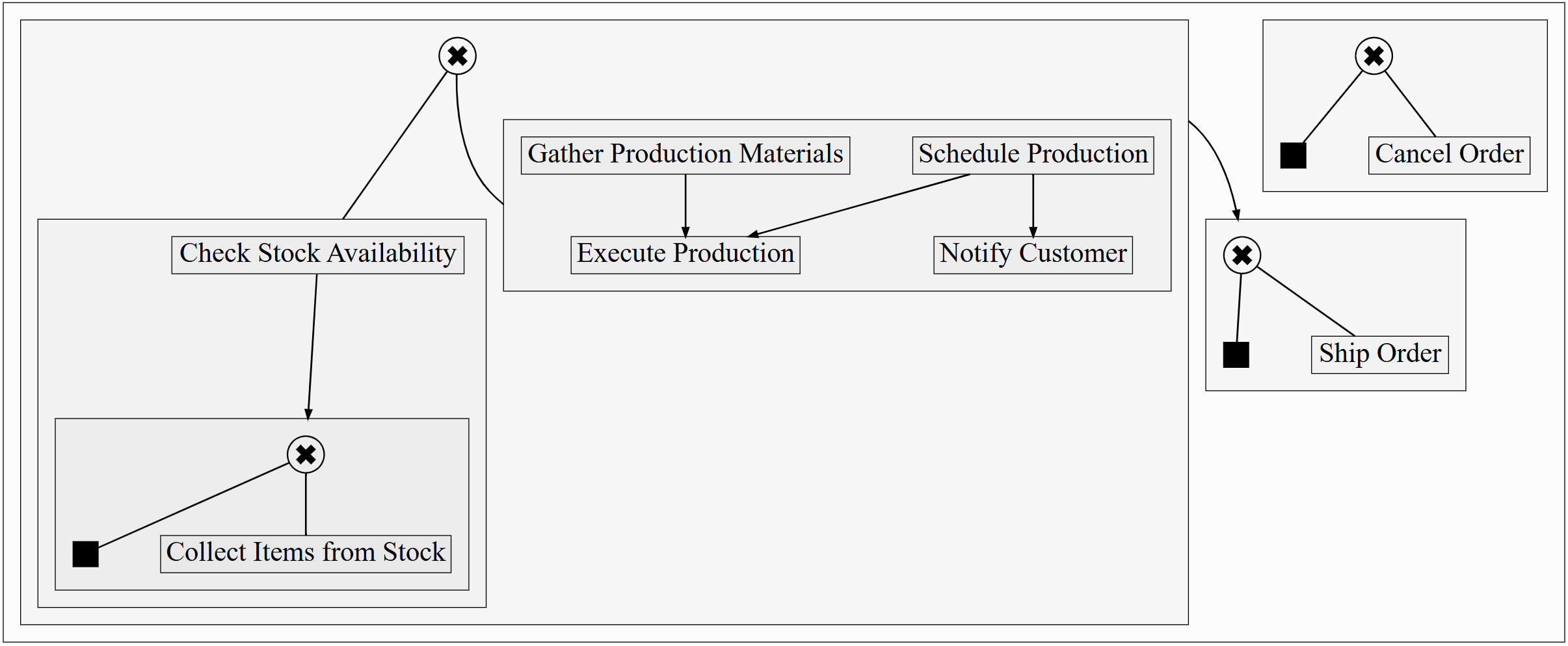}
        \caption{POWL model discovered with the POWL Inductive Miner \cite{DBLP:conf/icpm/KouraniSA23}.}
        \label{fig:powl}
    \end{subfigure}
    
    \caption{Process models for a retailer’s order fulfillment process. \label{fig:models1}}
\end{figure}

The \textit{Partially Ordered Workflow Language} (POWL) \cite{DBLP:conf/bpm/KouraniZ23} extends the capabilities of process trees. POWL relaxes the block-structure requirement for concurrency by using partial orders instead of the concurrency control-flow operator. A partial order specifies a set of children (i.e., sub-models) that must all be executed, but only imposes a partial ordering constraint on their execution, allowing for flexible interleavings. The Inductive Miner was extended to leverage POWL, allowing for the discovery of process models with non-block-structured concurrency components. For example, \autoref{fig:powl} shows a POWL model discovered using the POWL Inductive Miner from \cite{DBLP:conf/icpm/KouraniSA23} for our running example. Unlike the process tree, the discovered POWL model can precisely describe the dependencies between the four activities involved in the production sub-process. These dependencies are correctly described using a compact partial order with three ordering edges: Gather Production Materials $\rightarrow$ Execute Production, Schedule Production $\rightarrow$ Execute Production, Schedule Production $\rightarrow$ Notify Customer.

While POWL eliminates the block-structure limitation for concurrency, this limitation remains for decision points (represented by the XOR operator in both process trees and standard POWL models). In our running example in \autoref{fig:bpmn}, there is an initial choice between following the in-stock orders path or the production path. Then, within the in-stock orders path, there is another decision: either cancel the order if the item is out of stock, or join the other path leading to the shipping step. Both the base Inductive Miner and the POWL Inductive Miner fail to model this behavior precisely due to their lack of support for non-block-structured decision points. For instance, the discovered process tree (cf. \autoref{fig:tree}) and POWL model (cf. \autoref{fig:powl}) incorrectly allow the order to be canceled and shipped simultaneously.

Directly-Follows Graphs (DFGs) provide an alternative perspective on process behavior. A DFG is a simple graph where nodes represent activities, and edges represent direct succession relationships. DFGs are computationally efficient to construct from event logs, and they excel at naturally representing complex, non-block-structured branching decisions. For example, the DFG shown in \autoref{fig:dfg} accurately captures the non-block-structured decisions in our process. However, the presence of concurrency in the production sub-process introduces cycles into the DFG, which makes it imprecise. For instance, the DFG allows for skipping the activity ``Execute Production'' or repeating the sequence (``Gather Production Materials'' $\rightarrow$ ``Notify Customer") multiple times.

\begin{figure}[!t]
\centering
    
     \begin{subfigure}{0.39\textwidth}
        \centering  \includegraphics[width=0.95\textwidth]{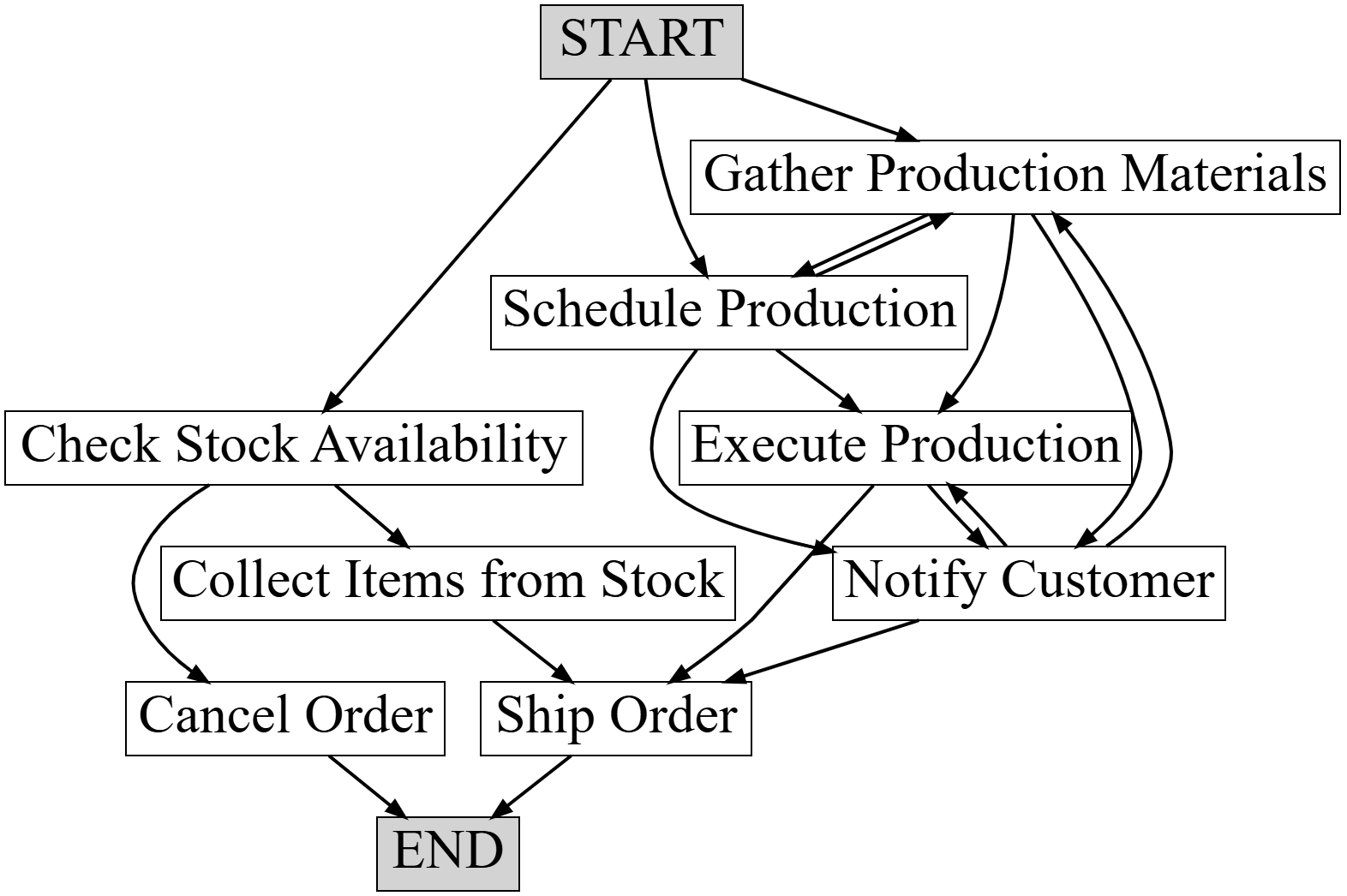}
        \caption{Directly-follows graph.}
        \label{fig:dfg}
    \end{subfigure}
    \begin{subfigure}{0.60\textwidth}
        \centering      \includegraphics[width=\textwidth]{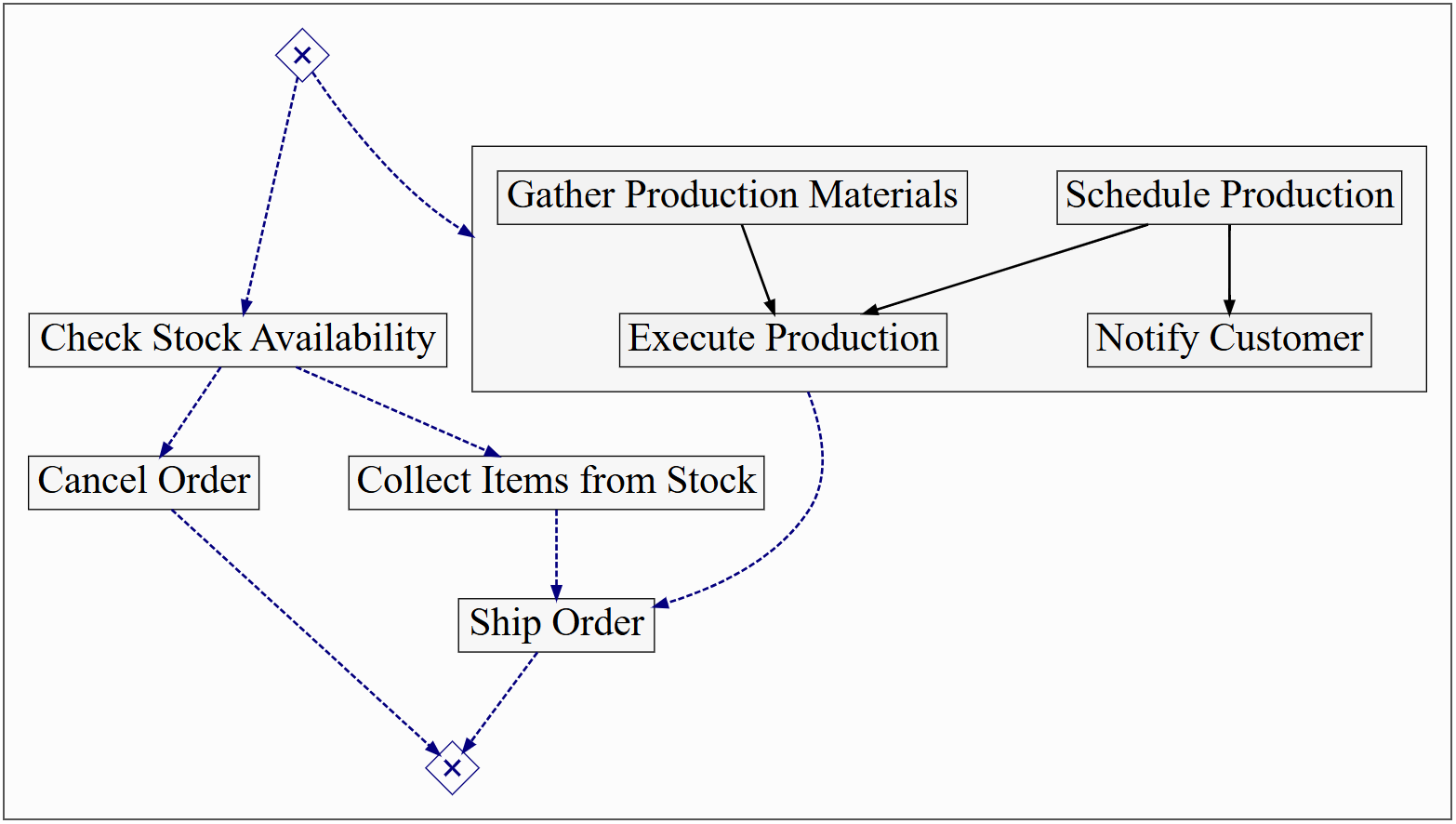}
        \caption{POWL 2.0 model.}
        \label{fig:powl20}
    \end{subfigure}
    
    \caption{Process models for a retailer’s order fulfillment process. \label{fig:models2}}
\end{figure}

Our example highlights the strengths of DFGs in modeling decision paths and their tendency to overgeneralize behavior when cycles are present. In a DFG, cycles arise from concurrency or repeated behaviors. POWL provides dedicated, semantically precise operators for these structures (partial orders and loops). This motivates the integration of a graph-based representation similar to DFGs for handling choices within the hierarchical structure of POWL, leveraging the strengths of both approaches while mitigating their individual limitations. Building upon this idea, this paper introduces POWL 2.0, an extension of POWL that enhances its capabilities to model non-block-structured decisions. Specifically, POWL 2.0 replaces the XOR operator with a more expressive construct: the \emph{choice graph}. Inspired by DFGs, choice graphs provide a structured way to represent complex, non-block-structured choices within POWL’s hierarchical framework while preserving desirable properties such as soundness.

\autoref{fig:powl20} illustrates a POWL 2.0 model discovered for our running example. The edges of choice graphs are visualized using blue dashed arcs to distinguish them from the solid black arcs in partial orders. In this process model, a choice graph captures the initial choice between the in-stock orders path and the production path, as well as the subsequent choice within the in-stock orders path—either to cancel the order or to join the other path and proceed to the shipment step. The production path is modeled using a partial order, correctly representing the dependencies between its activities. This example demonstrates POWL 2.0’s ability to accurately capture intricate choice behavior that process trees and standard POWL models struggle with, while retaining precise concurrency handling through partial orders.

The remainder of this paper is structured as follows. \autoref{sec:rel} discusses related work, and \autoref{sec:pre} introduces necessary preliminaries. In \autoref{sec:powl2}, we formally define POWL 2.0 and discuss its quality guarantees. In \autoref{sec:disc}, we present a discovery algorithm that extends the Inductive Miner to handle choice graphs, and we prove that this approach preserves the fitness guarantee of the Inductive Miner. \autoref{sec:ev} presents an experimental evaluation of the proposed approach. \autoref{sec:conc} concludes the paper and discusses future work.

\section{Related Work}\label{sec:rel}

A wide range of process discovery techniques has been developed. For a comprehensive overview, we refer to \cite{discovery2} and \cite{discovery}. Approaches such as the Inductive Miner \cite{sander} and the Evolutionary Tree Miner \cite{buijs2012genetic} generate block-structured process models, striking a balance between model quality and simplicity. In \cite{DBLP:conf/bpm/KouraniZ23}, POWL models were introduced, enabling the discovery of non-block-structured concurrency with the Inductive Miner. Further adaptations of the Inductive Miner for discovering POWL models were proposed in \cite{DBLP:conf/icpm/KouraniSA23} and \cite{DBLP:journals/is/KouraniZSA25}.

Several approaches discover process models directly in more flexible notations, such as Petri nets and BPMN. For instance, the Split Miner \cite{splitMiner} constructs a BPMN model by selectively adding split and join gateways based on log dependencies and frequencies. While the Split Miner can capture complex, non-block-structured dependencies, it does not guarantee soundness. Region-based mining techniques (e.g., \cite{bergenthum2007process} and \cite{carmona2008region}) generate Petri nets without enforcing block-structure constraints. However, these methods are sensitive to noise and can produce overly complex models or suffer from state-space explosion when applied to large event logs, limiting their scalability.

Approaches for directly translating DFGs into Petri nets have been explored (e.g., \cite{DBLP:conf/icpm/LeemansPW19}).
In \cite{POwithChoice2} and \cite{POwithChoice1}, methods are proposed for discovering prime event structures that integrate conflict relations into partially ordered graphs. These conflict relations offer a flexible way to represent non-block-structured decisions, improving the expressiveness of the resulting models.

The idea of combining different modeling methods to leverage their respective strengths has been explored in various contexts. In \cite{DBLP:conf/otm/LeemansTH18}, a discovery approach is proposed that combines multiple tree-based techniques. Furthermore, Petri nets are extended with additional causal edges in \cite{hybridPN2}, while the approach in \cite{DBLP:conf/otm/SlaatsSMR16} combines imperative process models with declarative constraints.

\section{Preliminaries}\label{sec:pre}
This section presents fundamental notations that will be used throughout the paper.

\subsection{Basic Notation}
A sequence over a set \( X \) is an ordered list of elements from \( X \), expressed as \( \sigma = \langle \sigma(1), \dots, \sigma(n) \rangle \). The length of \( \sigma \) is denoted by \( \card{\sigma} = n \), and the set of all sequences over a set \( X \) is denoted by \( X^* \). Given two sequences \( \sigma_1 \) and \( \sigma_2 \), their concatenation is written as \( \sigma_1 \cdot \sigma_2 \), for instance, \( \langle x_1 \rangle \cdot \langle x_2, x_1 \rangle = \langle x_1, x_2, x_1 \rangle \). For two sets of sequences \( L_1 \) and \( L_2 \), we define their concatenation as \( L_1 \cdot L_2 = \{\sigma_1 \cdot \sigma_2 \mid \sigma_1 \in L_1, \ \sigma_2 \in L_2\} \). The projection of a sequence \( \sigma \) onto a set \( Y \) is denoted by \( \sigma{\upharpoonright}_Y \), for example, \( \langle x_{1}, x_{2}, x_{1} \rangle {\upharpoonright}_{\{x_{1}, x_{3}\}} = \langle x_{1}, x_{1} \rangle \).

A \emph{multi-set} extends the concept of a set by keeping track of the multiplicities of its elements. A multi-set over a set \( X \) is represented as \( M = [{x_{1}}^{c_1}, \dots, {x_{n}}^{c_n}] \), where \( x_1, \dots, x_n \in X \) are the elements of \( M \) (denoted as \( x_i \in M \) for \( 1 \leq i \leq n \)), and \( M(x_i) = c_i \geq 1 \) specifies the multiplicity of \( x_i \) for each \( 1 \leq i \leq n \). We use \( \bag(X) \) to denote the set of all multi-sets over a set \( X \).

For a set \( X \), a \emph{partition} of \( X \) of size \( n \geq 1 \) is a set of subsets \( P = \{X_1, \dots, X_n\} \) such that \( X = X_1 \cup \dots \cup X_n \), \( X_i \neq \emptyset \), and \( X_i \cap X_j = \emptyset \) for all \( 1 \leq i < j \leq n \). For any \( x \in X \), we use \( P_x \) to denote the subset in the partition (also called a \emph{part}) that contains \( x \), i.e., \( P_x \in \{X_1, \dots, X_n\} \) and \( x \in P_x \). For example, consider the partition \( P = \{\{a, b\}, \{c\}\} \) of the set \( \{a, b, c\} \). Then, we have \( P_a = P_b = \{a, b\} \) and \( P_c = \{c\} \).

Let \( \po \subseteq X \times X \) be a binary relation over a set \( X \). We write \( x_1 \po x_2 \) to denote \( (x_1, x_2) \in \po \) and \( x_1 \notpo x_2 \) to indicate \( (x_1, x_2) \notin \po \). We define the \emph{transitive closure} of $\po$ as $\closure\po = \{(x, y) \mid \exists_{x_1, \dots, x_n \in X} \ x = x_1 \ \wedge \ y = x_n \ \wedge \forall_{1\leq i<n} x_i \po x_{i+1}\}$.

A \emph{strict partial order} (or simply \emph{partial order}) over a set \( X \) is a binary relation that satisfies \emph{irreflexivity} (\( x \notpo x \) for all \( x \in X \)) and \emph{transitivity} (\( x_1 \po x_2 \wedge x_2 \po x_3 \Rightarrow x_1 \po x_3 \)). These properties together imply \emph{asymmetry} (\( x_1 \po x_2 \Rightarrow x_2 \notpo x_1 \)). For \( n \geq 2 \), \( \Orders{n} \) denotes the set of all partial orders over \( \{1, \dots, n\} \). Given a set \( X = \{x_1, \dots, x_n\} \) with \( n \geq 2 \) and a partial order \( \po \in \Orders{n} \), we use \( \po(x_1, \dots, x_n) \) to denote the partial order \( \po' \) over \( X \) defined as follows: $i \po j \Leftrightarrow x_i \po' x_j \text{ for all } i, j \in \{1, \dots, n\}$.

Let \( \sigma_1, \dots, \sigma_n \in X^* \) be sequences over a set \( X \) with \( n \geq 2 \), and let \( \po \in \Orders{n} \). The \emph{order-preserving shuffle operator} \( \shuffle_{\po} \) produces the set of sequences obtained by interleaving \( \sigma_1, \dots, \sigma_n \) while maintaining both the partial order \( \po \) among the sequences and the inherent sequential order within each sequence. For example, consider the sequences \( \sigma_1 = \langle a, b\rangle \), \( \sigma_2 = \langle c\rangle \), \( \sigma_3 = \langle d, e\rangle \), and the partial order \( \po = \{(1, 2), (1, 3)\} \in \Orders{3} \). Then, the set of valid interleavings is $\shuffle_{\po}(\sigma_1, \sigma_2, \sigma_3) = \{\langle a, b, c, d, e\rangle, \langle a, b, d, c, e\rangle, \langle a, b, d, e, c\rangle\}$.

\subsection{Event Log}
We use \( \ActUniverse \) to denote the set of all activities. Additionally, we introduce \( \tau \notin \ActUniverse \) to denote the \emph{silent activity}, which is often used to model choices between executing or skipping a path in a process model.

An \emph{event log} \( L \in \bag(\ActUniverse^*)\) is a multi-set of sequences of activities. A \emph{trace} \( \sigma \in L \) represents the execution of a single process instance as a sequence of activities. Given an event log \( L \), we define the following notations:

\begin{itemize}
    \item \( \Sigma_{L} = \{a \in \sigma \mid \sigma \in L \} \) is the set of activities appearing in \( L \).
    
    \item \( \startL{L} = \{\sigma(1) \mid \sigma \in L\} \) is the set of \emph{start activities} in \( L \).

    \item \( \endL{L} = \{\sigma(\card{\sigma}) \mid \sigma \in L \} \) is the set of \emph{end activities} in \( L \).

    \item \( \dfg \subseteq \Sigma_{L} {\times} \Sigma_{L} \) is the \emph{Directly-Follows Graph (DFG) relation}, defined as follows: 
    \[
    a \dfg b \quad \text{iff} \quad \exists_{\sigma \in L, 1 \leq i < \card{\sigma}} \ \sigma(i) = a \ \wedge \ \sigma(i+1) = b.
    \]

\end{itemize}

For example, consider the event log \( L_1 =  [\langle a,b,c \rangle^3, \langle a,b,d \rangle^2] \), which consists of five traces. The corresponding sets are \( \Sigma_{L_1} =  \{a,b,c,d\} \), \( \startL{L_1} =  \{a\} \), and \( \endL{L_1} = \{c, d\} \). The directly-follows graph of \( L_1 \) is \( \dfg = \{(a,b), (b,c), (b,d)\} \).

\section{Extending POWL with Choice Graphs}\label{sec:powl2}
In this section, we introduce a new class of models that extends POWL, enabling a more expressive representation of branching behavior. Furthermore, we discuss the guarantees and limitations of this new class.

A POWL model \cite{DBLP:conf/bpm/KouraniZ23} is constructed recursively from a set of activities, combined using partial orders and the control flow operators \( \xor \) and \( \Loop \). The operator \( \xor \) represents an exclusive choice between submodels, whereas \( \Loop \) captures cyclic behavior between two submodels: the \textit{do-part} is executed first, and each execution of the \textit{redo-part} is followed by another execution of the do-part. In a partial order, all submodels are executed while preserving the specified execution order.

To enable non-block-structured decisions in POWL, we introduce \emph{choice graphs}, a structured approach to modeling exclusive choice paths in processes. A choice graph consists of a set of nodes connected by directed edges such that each node lies on a path from a designated \emph{start node} to a designated \emph{end node}.

\begin{definition}[Choice Graph]\label{def:choice_graph}
A \emph{choice graph} over a set of nodes $X$ is a tuple $G = (N, E)$ where:
\begin{itemize}
    \item $N = X \cup \{\xorstart{N}, \xorend{N}\}$ with two artificial start and end nodes $\xorstart{N}, \xorend{N} \notin X$.
    \item $E \subseteq N \times N$ is binary relation over $N$.
    \item $\xorstart{N}$ is the unique start node, i.e., $\{\xorstart{N}\} = \{x \in N \ | \ (N \times \{x\}) \cap E = \emptyset\}$.
    \item $\xorend{N}$ is the unique end node, i.e., $\{\xorend{N}\} = \{x \in N \ | \ ( \{x\} \times N) \cap E = \emptyset\}$.
    \item Every node is on a connected path from $\xorstart{N}$ to $\xorend{N}$.

\end{itemize}
\end{definition}
We use $\allDG(X)$ to denote the set of all choice graphs over a set $X$. A choice graph represents a collection of possible execution paths, each beginning at the start node and ending at the end node. Formally, for a choice graph $G = (N, E) \in \allDG(X)$ over a set $X$, we define the set of all paths in $G$ as follows:
\[
\paths{G} = \{ \langle x_1, ..., x_n \rangle \in X^* \ | \ (\xorstart{N}, x_1), (x_1, x_2), ..., (x_{n-1}, x_n), (x_n, \xorend{N}) \in E\}.
\]

With choice graphs established, we now define \emph{POWL 2.0}, an extended version of POWL that supports more complex decision-making structures while preserving its fundamental principles. The key difference from standard POWL is that exclusive choices between submodels are now represented using choice graphs instead of the \xor-operator.

\begin{definition}[POWL 2.0]\label{def:ext_powl}
A POWL model is recursively defined as follows:
\begin{itemize}
    \item An activity $a \in \ActUniverse \cup \{\tau\}$ is a POWL model.
    \item Let $\SetPOWL =  \{\powl_1, ..., \powl_n\}$ be a set of ${n \geq 2}$ POWL models.
    
    \begin{itemize}
        \item $\Loop(\powl_1, \powl_2)$ is a POWL model.
        \item A choice graph $G \in \allDG(\SetPOWL)$ is a POWL model.
        \item For a partial order $\po \in \Orders{n}$, $\po(\powl_1, ..., \powl_n)$ is a POWL model.
    \end{itemize}
\end{itemize}
\end{definition}

Similar to the original POWL framework \cite{DBLP:conf/bpm/KouraniZ23}, we define the semantics of a POWL 2.0 model recursively based on the semantics of its operators. The introduction of choice graphs alters how choices are represented. The language of a choice graph is defined as the set of sequences obtained by concatenating sequences from the languages of the submodels along a valid path in the choice graph.

\begin{definition}[POWL 2.0 Semantics]\label{def:ext_sem}
The language of a POWL model is recursively defined as follows:
\begin{itemize}

    \item ${\lang(a) = \{\langle a \rangle\}}$ for $a \in \ActUniverse$.
    \item ${\lang(\tau) = \{\langle \rangle\}}$.
    \item $\lang(\Loop(\powl_1, \powl_2)) = \lang(\powl_1) \cdot (\lang(\powl_2) \cdot \lang(\powl_1))^*$.
    \item Let $\SetPOWL =  \{\powl_1, ..., \powl_n\}$ be a set of ${n \geq 2}$ POWL models.
    
    \begin{itemize}     
         \item For $\po \in \Orders{n}$, $\lang(\po(\powl_1, ..., \powl_n)) = \{\sigma \in \shuffle_{\po}(\sigma_1 , ..., \sigma_n) \ | \ \forall_{1 \leq i \leq n} \ \sigma_i \in \lang(\powl_i) \}$.   

         \item For $G \in \allDG(\SetPOWL)$, $\lang(G) = \bigcup_{\path \in \paths{G}} \lang(\path_1) \cdot ... \cdot \lang(\path_{\card{\path}})$.
    \end{itemize}
\end{itemize}
\end{definition}

Note that POWL models can be defined over \emph{transitions} to allow for duplicating activities (i.e., each transition is considered as an instance of an activity). However, we assume a one-to-one mapping between transitions and activities in this paper and we define POWL models over activities for simplicity.

\subsection*{Discussion: Conversion to WF-Nets, Guarantees, Limitations}
POWL 2.0 models can be systematically converted into equivalent Workflow Nets (WF-nets), which, in turn, can be translated into BPMN models. A WF-net \cite{DBLP:journals/topnoc/HeeSW13a} is a Petri net with a unique source place and a unique sink place, where every other place and transition lies on a path from the source to the sink. The conversion algorithm, extending the one for standard POWL from \cite{DBLP:journals/is/KouraniZSA25}, recursively translates each POWL 2.0 construct into a WF-net fragment. Translating a choice graph into a WF-net involves: (i) introducing a unique source and sink place, (ii) recursively converting each sub-model (i.e., node in the choice graph) into its corresponding WF-net, and (iii) adding silent transitions to connect the sub-nets according to the edges of the choice graph. Adapting the conversion algorithm to handle choice graphs preserves its key guarantees: language equivalence and soundness. Formal inductive proofs extending the proofs from \cite{DBLP:journals/is/KouraniZSA25} can be easily constructed.

POWL 2.0's expressiveness places it within the hierarchy of sound WF-net subclasses. A \emph{marked graph} is a Petri net where each place has at most one incoming arc and at most one outgoing arc. A \emph{state machine} is a Petri net where each transition has at most one incoming arc and at most one outgoing arc. Any sound marked graph WF-net can be represented in POWL 2.0 using a partial order, while any sound state machine WF-net can be represented using a choice graph. 

Although POWL 2.0 removes the block-structure requirements for both concurrency and decision components, allowing more complex dependencies within each of those structures, it still enforces a block-structure between different types of operators. In other words, POWL 2.0 does not support structures where decision and concurrency split/join points interleave arbitrarily.

\section{Discovery of POWL 2.0 Models}\label{sec:disc}
In this section, we propose an approach for the discovery of POWL 2.0 model from event logs.

\subsection{Inductive Miner}
The Inductive Miner \cite{sander} is a process discovery algorithm that constructs process trees from event logs in a recursive manner. This approach was extended to discover POWL models by adding support for partial orders. Various variants of the POWL Inductive Miner were proposed in \cite{DBLP:conf/bpm/KouraniZ23}, \cite{DBLP:conf/icpm/KouraniSA23}, and \cite{DBLP:journals/is/KouraniZSA25}, differing primarily in how partial order cuts are identified while maintaining the same overall methodology. Our new methodology can be integrated with any of these variants as the proposed changes do not affect the partial order cut detection step. In the remainder of the paper, we use the term POWL Inductive Miner (\powlMiner) to refer to the variant defined in \cite{DBLP:conf/icpm/KouraniSA23}.

The Inductive Miner follows a recursive, top-down approach. It attempts to identify a \emph{cut}, which involves detecting a behavioral pattern in the event log and partitioning the activities accordingly. The algorithm recognizes different types of cuts that correspond to one of the process tree or POWL operators (\xor, \Loop, \seq, \And, and partial orders). Once a cut is identified, the event log is projected onto the identified parts (i.e., activity subsets), generating smaller sub-logs. The same procedure is then recursively applied to each sub-log until a \emph{base case} is reached—either an empty event log or an event log containing a single activity. In such cases, the base case is trivially converted into a process tree or a POWL model.

If neither a base case nor a cut can be found, the Inductive Miner applies a \emph{fall-through} function, which ensures that the recursion continues. This function may, for example, place an activity that appears in every trace into concurrency with the remaining activities. For a comprehensive overview of the base Inductive Miner and its adaptation to POWL, we refer to \cite{sander} and \cite{DBLP:journals/is/KouraniZSA25}, respectively.

\subsection{Mining for Choice Graphs}
In this section, we extend \powlMiner\ to discover POWL 2.0 models. We achieve this extension by replacing the \xor-cut detection step with a new procedure for detecting choice graphs. This allows for a more expressive representation of choices in process models while preserving the general structure of the inductive mining approach.

To incorporate choice graphs into the discovery process, we introduce the notion of a \emph{choice graph cut}, which involves partitioning the activities and assigning a choice graph structure over the partition to model branching behavior.

\begin{definition}[Choice Graph Cut]
Let $L$ be an event log. A choice graph cut over $L$ is a tuple $(A, G)$ where $A = \{A_1, ... , A_n\}$ is a partition of $\Sigma_{L}$ into $n \geq 2$ parts and $G \in \allDG(A)$ is a choice graph over $A$.
\end{definition}

Not every choice graph cut provides a valid representation of an event log. To ensure correctness, we impose structural constraints that align the choice graph with the event log’s DFG. Additionally, we add an acyclicity requirement. This restriction aims at enabling a more precise representation for repetition and concurrency behaviors using POWL’s dedicated constructs—partial orders for concurrency and loop operators for cyclic behavior.

\paragraph{Notation.} For two subsets of activities $A, B \subseteq \Sigma_{L}$, we extend the notation for the DFG relation as follows:
    \[
    A \dfg B \quad \text{iff} \quad \exists_{a \in A, b \in B} \ a \dfg b.
    \]

\begin{definition}[Valid Choice Graph Cut]\label{def:DGcut:valid}
Let $C = (A=\{A_1, ... , A_n\}, G=(N, E))$ be a choice graph cut over an event log $L$. $C$ is valid if and only if for all $A_i, A_j \in \{A_1, ... , A_n\}$:

\begin{enumerate}
    \item $(A_i \dfg A_j\ \wedge \ A_i \neq A_j)  \Leftrightarrow (A_i, A_j) \in E$.
    \item $A_i \cap \startL{L} \neq \emptyset \iff (\xorstart{N}, A_i) \in E$.
    \item $A_i \cap \endL{L} \neq \emptyset \iff (A_i, \xorend{N}) \in E$.
    \item $\langle\rangle \in L \Leftrightarrow (\xorstart{N}, \xorend{N}) \in E$.
    \item $(A_i \closure{\dfg} A_j\ \wedge \ A_j \closure{\dfg} A_i) \Rightarrow A_i = A_j$.
\end{enumerate}
\end{definition}

The first condition connects DFG edges to the choice graph's edges. The second and third conditions ensure that subsets containing start/end activities connect to the artificial start/end nodes. The forth condition links empty traces in the event log to a direct edge from start to end in the choice graph. The last condition enforces that if two parts are mutually reachable in the DFG, they must be the same part, introducing an acyclicity requirement. 

Note that the requirements (1-4) of \autoref{def:DGcut:valid} uniquely define a binary relation for any given partitioning of activities. In other words, in order to detect a choice graph cut, we need to generate a partition of activities that conforms with the acyclicity requirement and contains at least two parts.

\autoref{alg:dg_detection} generates a candidate partition of activities for a choice graph cut that ensures acyclicity. It achieves this by (i) initially putting every activity in a single part and (2) only merging two parts if they contain activities that are mutually reachable in the DFG. If the returned partition $A$ consists of a single part, then no valid choice graph cut exists. Otherwise, the requirements of \autoref{def:DGcut:valid} define a unique choice graph over $A$.

\begin{algorithm}[!t]
\caption{Generating a candidate partition for a choice graph cut.}\label{alg:dg_detection}
\DontPrintSemicolon
\KwIn{An event log $L$.}
\KwOut{A partition $A$ of the activities of $L$.}
\SetKwFunction{FMain}{MineDG}
\SetKwProg{Fn}{Function}{:}{}
\Fn{\FMain{$L$}}{
    $A \leftarrow \{\{a\} \ | \ a \in \Sigma_{L}\}$  
    
    \For{$(a_1, a_2) \in \Sigma_{L} \times \Sigma_{L}$}{
        \If{$(a_1 \closure{\dfg} a_2\ \wedge \ a_2 \closure{\dfg} a_1)$}{
            $A \leftarrow A \setminus \{A_{a_1}, A_{a_2}\} \cup \{A_{a_1} \cup A_{a_2}\}$  
        }
    }
    \KwRet{$A$}
}
\end{algorithm}

 After identifying a valid choice graph cut, the event log $L$ is projected on the different parts, creating sub-logs. 

\begin{definition}[Choice Graph Projection]\label{def:DGcut:project}
Let $L$ event log $L$ and $A \subseteq \Sigma_{L}$ be a subset of its activities. The choice graph projection of $L$ onto $A$ is defined as $\projDG(L, A) = \{\sigma{\upharpoonright}_A \mid \sigma \in L \land \sigma{\upharpoonright}_A \neq \langle \rangle \}.$

\end{definition}

We use \( \powlMinerDG \) to denote our extended variant of \powlMiner\ that mines for valid choice graph cuts instead of \xor-cuts. If a valid choice graph cut \( (A=\{A_1, ... , A_n\}, G=(N, E)) \) over an event log $L$ is identified, \( \powlMinerDG \) proceeds as follows: 
\begin{enumerate}
    \item The event log $L$ is projected onto each part $A_i$, creating a sub-log $L_i = \projDG(L, A_i)$.

    \item The algorithm is recursively applied on each sub-log $L_i$, creating a POWL model $\powl_i$. 
    
    \item The algorithm maps each part $A_i$ into its corresponding POWL model $\powl_i$ in G, creating a new choice graph $G'$. 
    
    \item The algorithm returns $G'$.
\end{enumerate}

Several reduction rules can be applied to a POWL 2.0 model generated by \powlMinerDG. For example, pure sequential behavior can be discovered using both partial order cuts and choice graph cuts. To ensure consistency, all such sequences are converted into partial orders. Additionally, nodes in a choice graph that share the same preceding and succeeding nodes are grouped together to enhance understandability. These reduction rules do not alter the language of the generated POWL 2.0 model and are therefore considered optional.

To ensure that the discovered POWL 2.0 models adequately represent the given event log, we establish a fitness guarantee for \( \powlMinerDG \). Specifically, we prove that every trace in the event log is included in the language of the discovered model.

\begin{lemma}[Fitness Preservation for Valid Choice Graph Cuts]\label{lem:choice-graph-fitness}
Let $L$ be an event log and let $(A=(A_1, \dots, A_n), G = (N, E))$ be a valid choice graph cut over $L$.  Let $G'$ be the choice graph obtained by replacing each $A_i$ in $G$ with the POWL model $\powl_i = \powlMinerDG(\projDG(L, A_i))$. Assume that $\projDG(L, A_i) \subseteq \lang(\powl_i)$ for all $1 \leq i \leq n$. Then for all $\sigma \in L$, we have $\sigma \in \lang(G')$.
\end{lemma}

\begin{proof}
Let $\sigma \in L$. We need to show that $\sigma \in \lang(G')$. Recall that the language of the choice graph $G'$ is defined as:
\[
    \lang(G') = \bigcup_{\path'=\langle \powl_1, \dots, \powl_k\rangle \in \paths{G'}} \lang(\powl_1) \cdot ... \cdot \lang(\powl_k).
    \]
This can be rewritten as:
\[
    \lang(G') = \bigcup_{\path = \langle A_1, \dots, A_k\rangle \in \paths{G}} \lang(\powlMinerDG(\projDG(L, A_1))) \cdot ... \cdot \lang(\powlMinerDG(\projDG(L, A_k))).
    \]
We construct a suitable path $\path = \langle A_1, \dots, A_k\rangle \in \paths{G}$ such that $\sigma$ is the concatenation of traces from the languages of the sub-models along that path. We have two cases:

\begin{itemize}
    \item \textbf{Case $\sigma = \langle \rangle$:} By \autoref{def:DGcut:valid}, if the empty trace is in $L$, then $(\xorstart{N}, \xorend{N}) \in E$. This gives us a direct path representing $\langle \rangle$. Consequently, $\langle \rangle \in \lang(G')$.

        \item \textbf{Case $\sigma = \langle a_1, \dots, a_m \rangle$ with $m > 0$:} We build the corresponding path $\path = \langle A_1, \dots, A_k\rangle \in \paths{G}$ ($k \leq m$) as follows:

        \begin{enumerate}
            \item \textbf{Start node:} Since $a_1 \in \startL{L}$, by \autoref{def:DGcut:valid}, it follows that $(\xorstart{N}, A_{a_1}) \in E$. Thus, we start the path with $A_{a_1}$, i.e., $A_1 = A_{a_1}$.

            \item \textbf{Intermediate nodes:} For subsequent activities $a_j$ in $\sigma$ ($2 \leq j \leq m$):

                \begin{itemize}
                    \item \textbf{Case $A_{a_j} = A_{a_{j-1}}$:} We remain within the same part and do not extend the path.
                    
                    \item \textbf{Case $A_{a_j} \neq A_{a_{j-1}}$:}

                    $a_{j-1}$ is directly followed by $a_j$ in $\sigma$, implying $A_{a_{j-1}} \dfg A_{a_j}$. Therefore, we can conclude by the first condition of \autoref{def:DGcut:valid} that $(A_{a_{j-1}}, A_{a_j}) \in E$ holds, and hence we extend our path with $A_{a_j}$.

                \end{itemize}
            
            \item \textbf{End node:} Since $a_m \in \endL{L}$, \autoref{def:DGcut:valid} implies $(A_{a_m}, \xorend{N}) \in E$. This shows that $\path$ is a valid path from $\xorstart{N}$ to $\xorend{N}$ in $G$.

        \end{enumerate}

    Thus, for every $\sigma \in L$, we have constructed a corresponding path $\path = \langle A_1, \dots, A_k\rangle \in \paths{G}$. Now, decompose $\sigma$ into sub-traces $\sigma_1, \dots, \sigma_k$ by projecting it on the different path steps $A_1, \dots, A_k$, i.e., $\sigma_i = \sigma{\upharpoonright}_{A_i} \in \projDG(L, A_i)$. By the assumption of the lemma, each $\sigma_i \in \lang(\powl_i)$. Since a valid choice graph is cycle-free, no part can be visited more than once within a single path. Therefore, concatenating the sub-traces in accordance with the path $\path$ yields $\sigma$. Thus:
    \[
    \sigma = \sigma_1 \cdot ... \cdot \sigma_k \in \lang(\powl_1) \cdot ... \cdot \lang(\powl_k) \subseteq \lang({G'}).
    \]

\end{itemize}
\end{proof}

\begin{theorem}[Fitness Guarantee]\label{thm:fit:dg}
Let $L$ be an event log and $\powl = \powlMinerDG(L)$. Every trace $\sigma \in L$ is included in the model's language, i.e., $\sigma \in \lang(\powlMinerDG(L))$.
\end{theorem}

\begin{proof}
We prove this theorem by induction on the recursive application of the algorithm as it constructs the POWL model.

\noindent\textbf{Base Case:} If the algorithm detects a base case, then the theorem trivially holds \cite{DBLP:journals/is/KouraniZSA25}.

\noindent\textbf{Inductive Hypothesis:} Assume the theorem holds for smaller sub-logs. This assumption means that any sub-model generated by the recursive application of the algorithm includes all traces of the corresponding sub-log in its language.

\noindent\textbf{Inductive Step:} We distinguish two cases:
   
    \begin{enumerate}
    \item \textbf{If a standard POWL cut is detected or the fall-through function is invoked:} The theorem holds since these cuts and the fall-through function are fitness-preserving \cite{DBLP:journals/is/KouraniZSA25}.
    
    \item \textbf{If a valid choice graph cut $(A=(A_1, \dots, A_n), G)$ is detected:} The algorithm maps each part $A_i$ into its corresponding POWL model $\powl_i = \powlMinerDG(\projDG(L, A_i))$ in G, returning another choice graph $G'$. By the inductive hypothesis, each sub-trace $\sigma_i \in \projDG(L, A_i)$ is in the language of the recursively discovered sub-model, i.e., $\sigma_i \in \lang(\powlMinerDG(\projDG(L, A_i)))$. Therefore, we can apply \autoref{lem:choice-graph-fitness}, which directly gives us that for all $\sigma \in L$, $\sigma \in \lang(G')$.
\end{enumerate}

\noindent Thus, by induction, any trace $\sigma \in L$ is included in $\lang(\powlMinerDG(L))$.
\end{proof}


\section{Evaluation}\label{sec:ev}
We implemented \powlMinerDG\ as an extension to the standard POWL Inductive Miner (\powlMiner). We compare these two approaches, evaluating two variants of each: one without any noise filtering and another variant applying the DFG-based noise filtering approach proposed in \cite{sander}. We use a noise filtering threshold of $0.2$. The experiments are conducted using 17 real-life event logs from \url{https://data.4tu.nl/}. For event logs that include life-cycle information with completion events (i.e., events with the state ``complete''), we only consider these completion events.

We assess both runtime performance and the quality of the generated process models. Model quality is evaluated by converting the discovered models into WF-nets and then applying two standard conformance checking metrics: fitness and precision. Fitness measures how well the model reproduces the behavior recorded in the event log \cite{adriansyah2011conformance}, while precision assesses how strictly the model conforms to the observed behavior \cite{DBLP:journals/isem/AdriansyahMCDA15}. We use the fitness and precision implementations available in PM4Py \cite{berti2023pm4py} and set a timeout of six hours for each conformance checking computation. In addition to fitness and precision, we report the f-score, which is their harmonic mean.

\begin{table}[!t]
\centering
\caption{Evaluation results. Missing values indicate timeouts. Higher quality scores for \powlMinerDG\ compared to \powlMiner\ are highlighted in green, while lower scores are highlighted in red. Significant time improvements are also highlighted in green.\label{tab:ev}}

\smaller
\begin{tabular}{!{\vrule width 1.2pt}l!{\vrule width 1.2pt}cc|cc!{\vrule width 1.2pt}cc|cc|cc|cc!{\vrule width 1.2pt}}
\hlinewd{1.2pt}
\multicolumn{1}{!{\vrule width 1.2pt}c!{\vrule width 1.2pt}}{\multirow{3}{*}{Event Log}} &
\multicolumn{4}{c!{\vrule width 1.2pt}}{\textbf{No Filtering}} &
\multicolumn{8}{c!{\vrule width 1.2pt}}{\textbf{Noise Threshold t=0.2}} \\
 &
  \multicolumn{2}{c|}{Time (sec)} &
  \multicolumn{2}{c!{\vrule width 1.2pt}}{F-Score} & 
  \multicolumn{2}{c|}{Time (sec)} &
  \multicolumn{2}{c|}{Fitness} &
  \multicolumn{2}{c|}{Precision} &
  \multicolumn{2}{c!{\vrule width 1.2pt}}{F-Score}\\
 &
  \powlMiner &
  \powlMinerDG &
  \powlMiner &
  \powlMinerDG &
  \powlMiner &
  \powlMinerDG &
  \powlMiner &
  \powlMinerDG &
  \powlMiner &
  \powlMinerDG &
  \powlMiner &
  \powlMinerDG\\
  \hlinewd{1.2pt}
BPIC2012 &
  19.3 &
  18.6 &
  0.27 &
  \cellcolor[HTML]{C6EFCE}{\color[HTML]{006100} 0.28} &
  11.0 &
  \cellcolor[HTML]{C6EFCE}{\color[HTML]{006100} 1.9} &
  0.98 &
  \cellcolor[HTML]{FFC7CE}{\color[HTML]{9C0006} 0.89} &
  0.26 &
  \cellcolor[HTML]{C6EFCE}{\color[HTML]{006100} 0.47} &
  0.41 &
  \cellcolor[HTML]{C6EFCE}{\color[HTML]{006100} 0.62} \\
BPIC2013 - Cl. Problems &
  0.0 &
  0.1 &
  0.88 &
  0.88 &
  0.1 &
  0.1 &
  0.99 &
  0.99 &
  0.95 &
  0.95 &
  0.97 &
  0.97 \\
BPIC2013 - Incidents &
  0.4 &
  0.6 &
  0.77 &
  0.77 &
  0.5 &
  0.5 &
  0.93 &
  \cellcolor[HTML]{FFC7CE}{\color[HTML]{9C0006} 0.91} &
  0.72 &
  \cellcolor[HTML]{C6EFCE}{\color[HTML]{006100} 0.87} &
  0.81 &
  \cellcolor[HTML]{C6EFCE}{\color[HTML]{006100} 0.89} \\
BPIC2013 - Open Problems &
  0.0 &
  0.0 &
  0.95 &
  0.95 &
  0.0 &
  0.0 &
  0.83 &
  0.83 &
  0.91 &
  0.91 &
  0.87 &
  0.87 \\
BPIC2017 &
  15.2 &
  14.8 &
  0.45 &
  \cellcolor[HTML]{C6EFCE}{\color[HTML]{006100} 0.47} &
  15.2 &
  \cellcolor[HTML]{C6EFCE}{\color[HTML]{006100} 2.7} &
  0.99 &
  \cellcolor[HTML]{FFC7CE}{\color[HTML]{9C0006} 0.98} &
  0.32 &
  \cellcolor[HTML]{C6EFCE}{\color[HTML]{006100} 0.55} &
  0.48 &
  \cellcolor[HTML]{C6EFCE}{\color[HTML]{006100} 0.70} \\
BPIC2017 - Offer Log &
  0.4 &
  0.3 &
  0.91 &
  \cellcolor[HTML]{C6EFCE}{\color[HTML]{006100} 1.00} &
  0.5 &
  0.5 &
  1.00 &
  1.00 &
  0.90 &
  \cellcolor[HTML]{C6EFCE}{\color[HTML]{006100} 1.00} &
  0.94 &
  \cellcolor[HTML]{C6EFCE}{\color[HTML]{006100} 1.00} \\
BPIC2018 &
  606.2 &
  601.0 &
  $-$ &
  $-$ &
  192.5 &
  189.8 &
  $-$ &
  $-$ &
  $-$ &
  $-$ &
  $-$ &
  $-$ \\
BPIC2019 &
  200.3 &
  198.5 &
  $-$ &
  $-$ &
  287.3 &
  \cellcolor[HTML]{C6EFCE}{\color[HTML]{006100} 20.2} &
  $-$ &
  $-$ &
  $-$ &
  $-$ &
  $-$ &
  $-$ \\
BPIC2020 - Dom. Decl. &
  0.3 &
  0.3 &
  0.56 &
  \cellcolor[HTML]{C6EFCE}{\color[HTML]{006100} 0.63} &
  0.1 &
  0.1 &
  0.93 &
  \cellcolor[HTML]{FFC7CE}{\color[HTML]{9C0006} 0.90} &
  0.38 &
  \cellcolor[HTML]{C6EFCE}{\color[HTML]{006100} 0.62} &
  0.54 &
  \cellcolor[HTML]{C6EFCE}{\color[HTML]{006100} 0.74} \\
BPIC2020 - Int. Decl. &
  2.3 &
  2.3 &
  0.24 &
  \cellcolor[HTML]{C6EFCE}{\color[HTML]{006100} 0.31} &
  0.7 &
  0.4 &
  0.88 &
  \cellcolor[HTML]{FFC7CE}{\color[HTML]{9C0006} 0.84} &
  0.34 &
  \cellcolor[HTML]{C6EFCE}{\color[HTML]{006100} 0.58} &
  0.50 &
  \cellcolor[HTML]{C6EFCE}{\color[HTML]{006100} 0.68} \\
BPIC2020 - Permit Log &
  22.0 &
  21.9 &
  0.16 &
  \cellcolor[HTML]{C6EFCE}{\color[HTML]{006100} 0.17} &
  3.7 &
  0.9 &
  0.79 &
  \cellcolor[HTML]{C6EFCE}{\color[HTML]{006100} 0.81} &
  0.19 &
  \cellcolor[HTML]{C6EFCE}{\color[HTML]{006100} 0.63} &
  0.31 &
  \cellcolor[HTML]{C6EFCE}{\color[HTML]{006100} 0.71} \\
BPIC2020 - Prepaid Travel &
  1.4 &
  1.4 &
  0.22 &
  0.22 &
  0.3 &
  0.1 &
  0.87 &
  \cellcolor[HTML]{C6EFCE}{\color[HTML]{006100} 0.88} &
  0.42 &
  \cellcolor[HTML]{C6EFCE}{\color[HTML]{006100} 0.74} &
  0.56 &
  \cellcolor[HTML]{C6EFCE}{\color[HTML]{006100} 0.80} \\
BPIC2020 - Req. Payment &
  0.2 &
  0.2 &
  0.45 &
  \cellcolor[HTML]{C6EFCE}{\color[HTML]{006100} 0.51} &
  0.1 &
  0.1 &
  0.91 &
  \cellcolor[HTML]{FFC7CE}{\color[HTML]{9C0006} 0.89} &
  0.39 &
  \cellcolor[HTML]{C6EFCE}{\color[HTML]{006100} 0.47} &
  0.54 &
  \cellcolor[HTML]{C6EFCE}{\color[HTML]{006100} 0.61} \\
Hospital Billing &
  2.1 &
  2.0 &
  0.67 &
  0.67 &
  2.0 &
  1.2 &
  0.96 &
  \cellcolor[HTML]{C6EFCE}{\color[HTML]{006100} 0.97} &
  0.56 &
  \cellcolor[HTML]{C6EFCE}{\color[HTML]{006100} 0.93} &
  0.71 &
  \cellcolor[HTML]{C6EFCE}{\color[HTML]{006100} 0.95} \\
Receipt Phase WABO &
  0.5 &
  0.5 &
  0.28 &
  \cellcolor[HTML]{C6EFCE}{\color[HTML]{006100} 0.30} &
  0.1 &
  0.0 &
  0.81 &
  \cellcolor[HTML]{C6EFCE}{\color[HTML]{006100} 0.88} &
  0.25 &
  \cellcolor[HTML]{C6EFCE}{\color[HTML]{006100} 0.55} &
  0.38 &
  \cellcolor[HTML]{C6EFCE}{\color[HTML]{006100} 0.67} \\
Road Traffic Fine &
  1.0 &
  0.9 &
  0.74 &
  0.74 &
  0.9 &
  0.9 &
  0.87 &
  \cellcolor[HTML]{C6EFCE}{\color[HTML]{006100} 0.98} &
  0.78 &
  \cellcolor[HTML]{C6EFCE}{\color[HTML]{006100} 1.00} &
  0.82 &
  \cellcolor[HTML]{C6EFCE}{\color[HTML]{006100} 0.99} \\
Sepsis Cases &
  1.2 &
  1.2 &
  0.44 &
  0.44 &
  0.5 &
  0.2 &
  0.91 &
  0.91 &
  0.40 &
  \cellcolor[HTML]{C6EFCE}{\color[HTML]{006100} 0.54} &
  0.56 &
  \cellcolor[HTML]{C6EFCE}{\color[HTML]{006100} 0.68}\\
  \hlinewd{1.2pt}
\end{tabular}
\end{table} 

All discovered process models are available at \url{https://zenodo.org/records/15050973}. The evaluation results are summarized in \autoref{tab:ev}. For the case with no filtering, we only report the f-scores as quality indicators since all discovered models achieved perfect fitness, which is guaranteed by both \powlMiner\ and \powlMinerDG.

Overall, we observe that the introduction of choice graphs has no negative impact on runtime performance. In cases without noise filtering, there is no significant difference in execution time between \powlMiner\ and \powlMinerDG. However, with the noise filtering threshold of 0.2, \powlMinerDG\ exhibited improved runtime performance for some event logs. For instance, for the BPIC2019 event log, the execution time decreased from 287 seconds with \powlMiner\ to 20 seconds with \powlMinerDG. This demonstrates that mining for choice graphs does not compromise the inherent scalability of the Inductive Miner; in fact, it can improve runtime performance when valid choice graphs are detected.

Regarding quality, we observe that \powlMinerDG\ achieved equal or higher precision and f-score compared to \powlMiner\ in all cases. In 21 out of the 30 cases with no timeouts, \powlMinerDG\ led to more precise models. In some instances, fitness dropped slightly, but these decreases were compensated by notable increases in precision, leading to consistently higher f-scores. In other cases, such as the Receipt Phase WABO log, mining for decision graphs with noise filtering led to improvements in both fitness (from 0.81 to 0.88) and precision (from 0.25 to 0.55).

\autoref{fig:discoveredModels} presents the process models discovered for the BPIC2017 - Offer Log, converted into BPMN. In this case, mining for choice graphs increased the f-score from 0.91 to 1.0. One notable improvement for this event log is that the model discovered by \powlMiner\ allows an offer to be accepted immediately after creation, while the more precise model generated by \powlMinerDG\ ensures that offers can only be accepted after they have been sent and subsequently returned. This example highlights the advantage of incorporating choice graphs in process discovery, leading to improved precision without compromising scalability.

\begin{figure}[!t]
\centering
    
     \begin{subfigure}{\textwidth}
        \centering  
        \includegraphics[width=0.9\textwidth]{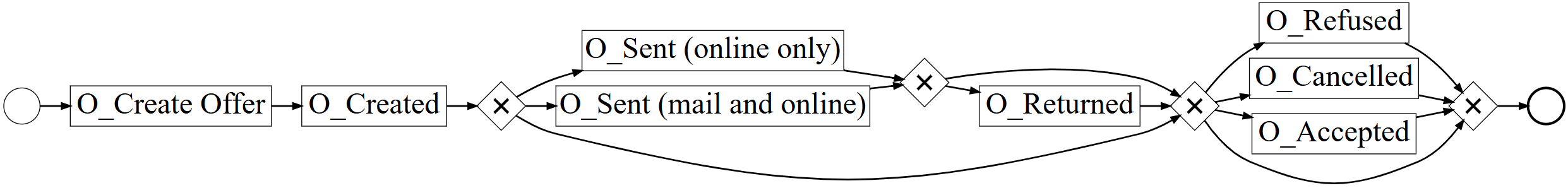}
        \caption{By \powlMiner.}
    \end{subfigure}
    
    \begin{subfigure}{\textwidth}
        \centering      
        \includegraphics[width=\textwidth]{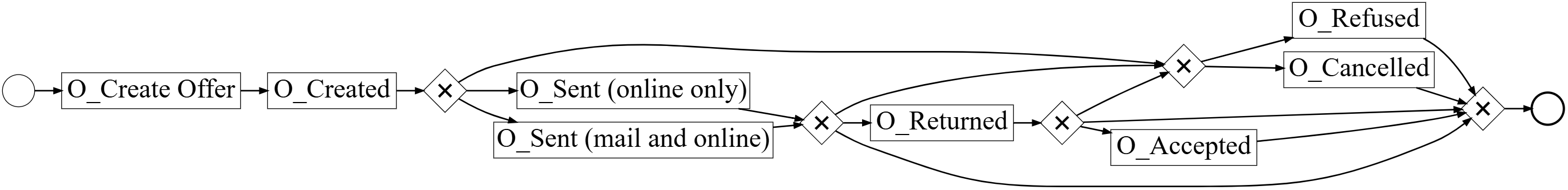}
        \caption{By \powlMinerDG.}
    \end{subfigure}
    
    \caption{Process models discovered for the BPIC2017 - Offer Log. \label{fig:discoveredModels}}
\end{figure}

\section{Conclusion} \label{sec:conc}
This paper introduced POWL 2.0, an extension of the Partially Ordered Workflow Language (POWL) designed to enable the modeling of non-block-structured decisions. By replacing the traditional exclusive choice operator with choice graphs, POWL 2.0 provides a more expressive and flexible way to represent complex branching logic within the hierarchical framework of POWL. We presented a discovery algorithm that extends the Inductive Miner to mine for acyclic choice graphs, and we formally proved that this algorithm preserves the fitness guarantee of the Inductive Miner. Our experimental evaluation demonstrated the ability of choice graphs to capture intricate decision-making patterns in real-world processes.

One direction for future work is to explore further extensions to POWL to capture more complex process structures, while carefully considering the trade-offs between expressiveness, scalability, and the preservation of formal guarantees. While the current discovery approach mines only for acyclic choice graphs, POWL 2.0 itself allows for more general decision structures. Future research will explore the controlled integration of cyclic choice graphs in process discovery. Furthermore, we plan to investigate the integration of POWL with other process mining and business process management techniques beyond process discovery.

\bibliographystyle{splncs04}
\bibliography{lit}

\end{document}